\documentclass{article}

\usepackage{nomencl}
\makenomenclature

\usepackage{amsthm}
\usepackage{amssymb}
\usepackage[utf8]{inputenc}
\usepackage{amsmath}
\newtheorem{definition}{Definition}
\newtheorem{remark}{Remark}

\newtheorem{theorem}{Theorem}
\newtheorem{corollary}{Corollary}
\newtheorem{lemma}{Lemma}

\usepackage{graphicx} 

\usepackage{natbib}

\usepackage{algorithm}
\usepackage{algorithmic}

\usepackage{hyperref}

\newcommand{\impr}{I(a(t)=a^*(t))}

\newcommand{\reg}{\mbox{regret}}
\newcommand{\Ex}{\mathbb{E}}
\newcommand{\ExR}[1]{\mathbb{E}\left[ #1 \right]}
\newcommand{\PrR}[1]{\Pr\left( #1 \right)}
\newcommand{\condR}{\ \vline\  }

\newcommand{\cof}{\ell_t}
\newcommand{\cofvalue}{R\sqrt{d\ln\left(\frac{t^3}{\delta}\right)}+1}
\newcommand{\std}{v}
\newcommand{\stdt}{\std_t}
\newcommand{\stdvalue}{R\sqrt{9 d \ln(\frac{t}{\delta})}}
\newcommand{\stdvalueFixed}{R\sqrt{9d \ln(\frac{T}{\delta})}}
\newcommand{\logdev}{\min\{\sqrt{4d\ln(t)}, \sqrt{4\log(tN)}\ \}}

\newcommand{\probt}{p}
\newcommand{\probT}{p}
\newcommand{\pFixed}{\frac{1}{4e\sqrt{\pi}}}
\newcommand{\p}{\pFixed}

\newcommand{\unsat}{ unsaturated }
\newcommand{\sat}{ saturated }
\newcommand{\tdev}{g_t}
\newcommand{\tdevT}{g_T}
\newcommand{\comment}[1]{}
\newcommand{\addedShipra}[1]{#1}
\newcommand{\removedShipra}[1]{}
\newcommand{\EQ}[3]{\begin{center} \vspace{#1}$#3$ \vspace{#2}\end{center}}
\newcommand{\concise}[3]{\mbox{$#1$} & \mbox{$#2$}  & \mbox{$#3$}  \vspace{-0.05in}}
\newcommand{\Emu}{E^{\mu}}
\newcommand{\Etheta}{E^{\theta}}

\title{Thompson Sampling for Contextual Bandits with Linear Payoffs}
\author{Shipra Agrawal\\
Microsoft Research
\and
Navin Goyal \\
Microsoft Research
}
\begin{document} 
\maketitle




\begin{abstract}
Thompson Sampling is one of the oldest heuristics for multi-armed bandit problems. It is a randomized algorithm based on Bayesian ideas, and has recently generated significant interest after several studies demonstrated it to have better empirical performance compared to the state-of-the-art methods. However, many questions regarding its theoretical performance remained open. In this paper, we design and analyze a generalization of Thompson Sampling algorithm for the stochastic contextual multi-armed bandit problem with linear payoff functions, when the contexts are provided by an adaptive adversary. This is among the most important and widely studied version of the contextual bandits problem. We provide the first theoretical guarantees for the contextual version of Thompson Sampling. 
We prove a high probability regret bound of $\tilde{O}(d^{3/2}\sqrt{T})$ (or $\tilde{O}(d\sqrt{T \log(N)})$), which is the best regret bound achieved by any computationally efficient algorithm for this problem, and is within a factor of $\sqrt{d}$ (or $\sqrt{\log(N)}$) of the information-theoretic lower bound for this problem.
\end{abstract}
\newpage

\section{Introduction}
Multi-armed bandit (MAB) problems model the exploration/exploitation trade-off inherent in many sequential decision problems. There are many versions of multi-armed bandit problems; a
particularly useful version is the contextual multi-armed bandit problem.
In this problem, in each of $T$ rounds, a learner is presented with the choice of taking one out of $N$  
\nomenclature[N]{$N$}{number of arms} actions, referred to as $N$ arms. Before making the choice of which arm to play, the learner sees $d$-dimensional feature vectors $b_i$, referred to as ``context", associated with each arm $i$. The learner uses these feature vectors along with the feature vectors and rewards of the arms played by her in the past to make the choice of the arm to play in the current round. Over time, the learner's aim is to gather enough information about how the feature vectors and rewards relate to each other, so that she can predict, with some certainty, which arm is likely to give the best reward by looking at the feature vectors. The learner competes with a class of predictors, 
in which each predictor takes in the feature vectors and predicts which arm will give the best reward. If 
 the learner can guarantee to do nearly as well as the predictions of the best predictor in hindsight (i.e., have low regret), then the learner is said to successfully compete with that class. 
 
 In the contextual bandits setting with {\em linear payoff functions}, the learner competes with the class of all ``linear" predictors on the feature vectors. That is, a predictor is defined by a $d$-dimensional parameter $\overline{\mu} \in {\mathbb{R}}^d$, and the predictor ranks the arms according to $b_i^T\overline{\mu}$. We consider stochastic contextual bandit problem under linear realizability assumption, that is, we assume that there is an unknown underlying parameter $\mu \in \mathbb{R}^d$ such that the expected reward for each arm $i$, given context $b_i$, is $b_i^T\mu$. Under this realizability assumption, the linear predictor corresponding to $\mu$ is in fact the best predictor and the learner's aim is to learn this underlying parameter. This realizability assumption is standard in the existing literature on contextual multi-armed bandits, e.g.
\citep{DBLP:journals/jmlr/Auer02,  DBLP:conf/nips/FilippiCGS10, DBLP:journals/jmlr/ChuLRS11, 
 DBLP:conf/nips/Abbasi-YadkoriPS11}. 


Thompson Sampling (TS) is one of the earliest heuristics for multi-armed bandit problems. The first
version of this Bayesian heuristic is around 80 years old, dating to \citet{Thompson}.  Since then, it has been rediscovered numerous times independently in the context of reinforcement learning, e.g., in \citet{Wyatt97, OrtegaB10, DBLP:conf/icml/Strens00}. It is a member of the family of {\it randomized probability matching} algorithms. 
  The basic idea is to assume a simple prior distribution on the underlying parameters of the reward distribution of every arm, and at every time step, play an arm according to its posterior probability of being the best arm. 
The general structure of TS for the contextual bandits problem involves the following elements: 
\begin{enumerate}\setlength{\itemsep}{-1mm}
\item a set $\Theta$ of parameters $\tilde{\mu}$;
\item a prior distribution $P(\tilde{\mu})$ on these parameters;
\item past observations ${\cal D}$ consisting of (context $b$, reward $r$) for the past time steps;
\item a likelihood function $P(r|b, \tilde{\mu})$, which gives the probability of reward given a context $b$ and a parameter $\tilde{\mu}$;
\item a posterior distribution $P(\tilde{\mu}|{\cal D}) \propto P({\cal D} | \tilde{\mu})P(\tilde{\mu})$, where $P({\cal D} | \tilde{\mu})$ is the likelihood function.\vspace{-0.15in}
\end{enumerate}

In each round, TS plays an arm according to its posterior probability of having the best parameter.
A simple way to achieve this is to produce a sample of parameter for each arm, using the posterior distributions, and play the arm that produces the best sample. 
\addedShipra{In this paper, we design and analyze a natural generalization of Thompson Sampling (TS) for contextual bandits; this generalization fits the above general structure, and uses Gaussian prior and Gaussian likelihood function.}
We emphasize that although TS is a Bayesian approach, the description of the algorithm and our analysis apply to the prior-free stochastic MAB model, and our regret bounds will hold irrespective of whether or not the actual reward distribution matches the Gaussian likelihood function used to derive this Bayesian heuristic. Thus, our bounds for TS algorithm are directly comparable to the UCB family of algorithms which form a frequentist approach to the same problem. One could interpret the priors used by TS as a way of capturing the current knowledge about the arms.

  
Recently, TS has attracted considerable attention. Several studies (e.g., \citet{Granmo, Scott, GraepelCBH10, DBLP:conf/nips/ChapelleL11, MayL, KaufmannMunos12}) have empirically demonstrated the efficacy of TS: \citet{Scott} provides a detailed discussion of probability matching techniques in many general settings along with favorable empirical comparisons with other techniques. \citet{DBLP:conf/nips/ChapelleL11} demonstrate that for the basic stochastic MAB problem, empirically TS achieves regret comparable to the lower bound of \citet{LaiR}; and in applications like display advertising and news article recommendation modeled by the contextual bandits problem, it is competitive to or better than the other methods such as UCB. In their experiments, TS is also more robust to delayed or batched feedback 
than the other methods. TS has been used in an industrial-scale application for CTR prediction of search ads on search 
engines~\citep{GraepelCBH10}. \citet{KaufmannMunos12} do a thorough comparison of TS with the best 
known versions of UCB and show that TS has the lowest regret in the long run. 

However, the theoretical understanding of TS is limited.
\citet{Granmo} and \citet{MayKLL} provided weak guarantees, namely, a bound of $o(T)$ on the  expected regret in time $T$. For the the basic (i.e. without contexts) version of the stochastic MAB problem, some significant progress was made by \citet{AgrawalG12}, \citet{KaufmannMunos12} and, more recently, by \citet{arxiv:basicMAB}, who provided 
optimal regret bounds on the expected regret.      	 	 
But, many questions regarding theoretical analysis of TS remained open, including 
high probability regret bounds, and regret bounds for the more general contextual bandits setting. In particular, the contextual MAB problem 
does not seem easily amenable to the techniques used so far for analyzing TS for the basic MAB problem. In Section \ref{sec:formal-outline}, we describe some of these challenges. 
Some of these questions and difficulties were also formally raised as a COLT 2012 open problem~\citep{ChapelleL12}. 

In this paper, we use novel martingale-based analysis techniques to demonstrate that  TS \addedShipra{(i.e., our Gaussian prior based generalization of TS for contextual bandits)} achieves high probability, near-optimal 
regret bounds for stochastic contextual bandits with linear payoff functions. To our knowledge, ours are the first non-trivial regret bounds for TS for the contextual bandits problem. Additionally, our results are the first high probability 
regret bounds for TS, even in the case of basic MAB problem. 
This essentially solves the COLT 2012 open problem by\citep{ChapelleL12} for contextual bandits with linear payoffs.

We provide a regret bound of $\tilde{O}(d^{3/2}\sqrt{T})$, or $\tilde{O}(d\sqrt{T \log(N)})$ (whichever is smaller), upper bound on the regret for Thompson Sampling algorithm. Moreover, the Thomspon Sampling algorithm we propose is efficient (runs in time polynomial in $d$) to implement as long as it is efficient to optimize a linear function over the set of arms (see Section \ref{sec:formal-algo} paragraph ``Computational efficiency" for further discussion). Although the information theoretic lower bound for this problem is $\Omega(d\sqrt{T})$, an upper bound of $\tilde{O}(d^{3/2}\sqrt{T})$ is in fact the best achieved by any computationally efficient algorithm in the literature when number of arms $N$ is large (see the related work section \ref{sec:related} for a detailed discussion). To determine whether there is a gap between computational and information theoretic lower bound for this problem is an intriguing open question. 

Our version of Thompson Sampling algorithm for the contextual MAB problem, described formally in Section \ref{sec:formal-algo}, uses Gaussian prior and Gaussian likelihood functions. Our techniques \addedShipra{can be extended}\removedShipra{are easily extendable} to the use of other prior distributions, satisfying certain conditions, as discussed in Section \ref{sec:conclusions}.

\section{Problem setting and algorithm description}
\label{sec:formal}

\subsection{Problem setting} \label{sec:setting}
There are $N$ arms. At time $t = 1, 2, \ldots $, a context vector $b_i(t) \in {\mathbb R}^d$, is revealed for every arm $i$. 
\nomenclature[b]{$b_i(t)$}{context vector for arm $i$ at time $t$} \nomenclature[d]{$d$}{The dimension of context vectors}
These context vectors are chosen by an adversary in an adaptive manner after observing the arms played and their rewards up to time $t-1$, i.e. history ${\cal H}_{t-1}$,
\EQ{-0.1in}{-0.1in}{{\cal H}_{t-1} = \{a(\tau), r_{a(\tau)}(\tau), b_i(\tau), i=1,\ldots, N, \tau=1,\ldots, t-1\},}
\nomenclature[H]{${\cal H}_{t-1}$}{$= \{a(\tau), r_{a(\tau)}(\tau), b_i(\tau), i=1,\ldots, N, \tau=1,\ldots, t-1\}$}
where $a(\tau)$ denotes the arm played at time $\tau$. 
\nomenclature[a(t)]{$a(t)$}{The arm played at time $t$}
Given $b_i(t)$, the reward for arm $i$ at time $t$ is generated from an (unknown) distribution with mean $b_i(t)^T\mu$, where $\mu\in {\mathbb R}^d$ is a fixed but unknown parameter. 
\nomenclature[mu]{$\mu$}{The unknown $d$-dimensional parameter }
\nomenclature[r]{$r_i(t)$}{Reward for arm $i$ at time $t$}
 \EQ{-0.05in}{-0.1in}{\ExR{r_i(t) \condR \{b_i(t)\}_{i=1}^N, {\cal H}_{t-1}} = \ExR{r_i(t) \condR b_i(t)} = b_i(t)^T\mu.}
An algorithm for the {\em contextual bandit problem} needs to choose, at every time $t$, an arm $a(t)$ to play, using history ${\cal H}_{t-1}$ and current contexts $b_i(t), i=1,\ldots, N$.
Let $a^*(t)$ denote the optimal arm at time $t$, i.e. $a^*(t) = \arg \max_{i} b_i(t)^T\mu.$
\nomenclature[a*]{$a^*(t)$}{The optimal arm at time $t$}
And let $\Delta_{i}(t)$ be the difference between the mean rewards of the optimal arm and of arm $i$ at time $t$, i.e.,
\EQ{-0.05in}{0in}{\Delta_i(t) = b_{a^*(t)}(t)^T\mu -  b_{i}(t)^T\mu.}
Then, the regret at time $t$ is defined as
\EQ{-0.05in}{-0.05in}{\reg(t) =  \Delta_{a(t)}(t).}
\nomenclature[regret]{$\reg(t)$}{Regret at time $t$}
The objective is to minimize the total regret ${\cal R}(T) = \sum_{t=1}^T \reg(t)$ in time $T$. The time horizon $T$ is finite but possibly unknown.

We assume that $\eta_{i,t}=r_i(t)-b_i(t)^T\mu$ is conditionally $R$-sub-Gaussian for a constant $R\ge 0$, i.e.,
\EQ{-0.05in}{-0.05in}{\forall \lambda \in {\mathbb R}, \Ex[e^{\lambda\eta_{i,t}} | \{b_i(t)\}_{i=1}^N, {\cal H}_{t-1}] \le \exp\left(\frac{\lambda^2R^2}{2}\right).}
This assumption is satisfied whenever $r_i(t) \in [b_i(t)^T\mu-R, b_i(t)^T\mu+R]$
(see Remark 1 in Appendix A.1 of \citet{DBLP:conf/nips/FilippiCGS10}). 
\addedShipra{We will also assume that $||b_i(t)||\le 1$, $||\mu||\le 1$, and $\Delta_i(t) \le 1$ for all $i,t$ (the norms, unless otherwise indicated, are $\ell_2$-norms). These assumptions are required to make the regret bounds scale-free, and are standard in the literature on this problem.  
  If $||\mu||\le c, ||b_i(t)||\le c, \Delta_i(t)\le c$ instead, then our regret bounds would increase by a factor of $c$.}

\begin{remark}
\label{rem:regret}
An alternative definition of regret that appears in the literature is 
\EQ{-0.05in}{-0.05in}{{\reg(t)= r_{a^*(t)}(t) -  r_{a(t)}(t).}}
We can obtain the same regret bounds for this alternative definition of regret. The details are provided in the supplementary material in  Appendix \ref{app:proof}.
\end{remark}

\subsection{Thompson Sampling algorithm} 
\label{sec:formal-algo}


We use Gaussian likelihood function and Gaussian prior to design our version of Thompson Sampling algorithm.
More precisely, suppose that the {\bf likelihood} of reward $r_i(t)$ at time $t$, given context $b_i(t)$ and parameter $\mu$, were given by the pdf of Gaussian distribution ${\cal N}(b_i(t)^T\mu, \std^2)$. Here, $\std = \stdvalueFixed$.
Let  
\EQ{-0.05in}{-0.05in}{ B(t) = I_d + \sum_{\tau=1}^{t-1} b_{a(\tau)}(\tau) b_{a(\tau)}(\tau)^T}
\nomenclature[B]{$B(t)$}{$= I_d + \sum_{\tau=1}^{t-1} b_{a(\tau)}(\tau) b_{a(\tau)}(\tau)^T$}
\EQ{-0.05in}{-0.05in}{\hat{\mu}(t) = B(t)^{-1} \left(\sum_{\tau=1}^{t-1} b_{a(\tau)}(\tau)r_{a(\tau)}(\tau)\right).}
\nomenclature[muhat]{$\hat{\mu}(t)$}{$= B(t)^{-1} \left(\sum_{\tau=1}^{t-1} b_{a(\tau)}(\tau)r_{a(\tau)}(\tau)\right)$ (Empirical estimate of mean at time $t$) }
	Then, if the {\bf prior} for $\mu$ at time $t$ is given by ${\cal N}(\hat{\mu}(t), \std^2 B(t)^{-1})$, it is easy to compute 
	the {\bf posterior} distribution at time $t+1$,
	 \EQ{-0.1in}{-0.05in}{\Pr(\tilde{\mu} | r_i(t))  \propto \Pr(r_i(t) | \tilde{\mu}) \Pr(\tilde{\mu})}
	  as ${\cal N}(\hat{\mu}(t+1), \std^2 {B(t+1)}^{-1})$ 	(details of this computation are in Appendix \ref{app:posterior}).
In our Thompson Sampling algorithm, at every time step $t$, we will simply generate a sample $\tilde{\mu}(t)$ from the distribution ${\cal N}(\hat{\mu}(t), \std^2 {B(t)}^{-1})$, and play the arm $i$ that maximizes $b_i(t)^T\tilde{\mu}(t)$.

\addedShipra{
We emphasize that the Gaussian priors and the Gaussian likelihood model for rewards are only used above to design the Thompson Sampling algorithm for contextual bandits. Our analysis of the algorithm allows these models to be completely unrelated to the {\it actual} reward distribution. The assumptions on the actual reward distribution are only those mentioned in Section \ref{sec:setting}, i.e., the $R$-sub-Gaussian assumption.}

	
\begin{algorithm}[H] 
\caption{Thompson Sampling for Contextual bandits}
  \begin{algorithmic}
\FORALL{$t=1, 2,\ldots, $} 
		\STATE Sample $\tilde{\mu}(t)$ from distribution ${\cal N}(\hat{\mu}(t), \std^2 B(t)^{-1})$. \\
  	\STATE Play arm $a(t) := \arg \max_i b_i(t)^T\tilde{\mu}(t)$, and observe reward $r_{a(t)}(t)$.\\
\ENDFOR
  	\end{algorithmic}
\end{algorithm}
\comment{
We also consider a small modification of Thompson Sampling, where at time step $t$, we generate a sample $\tilde{\mu}(t)$ from distribution ${\cal N}(\hat{\mu}(t), \stdt^2 B(t)^{-1})$ where $\stdt=\stdvalue$. That is, instead of using a fixed $\epsilon$, we use different $\epsilon_t$ at different time steps $t$. As we discuss in the next section, using $\epsilon_t=\frac{1}{\ln t}$ will give us an optimal regret bound of $\tilde{O}(d^{3/2}\sqrt{T})$. Note that Algorithm 1 can always be recovered from Algorithm 2, by setting $\epsilon_t=\epsilon$ (i.e. $\stdt=\std$) for all $t$. However, the Bayesian posterior interpretation discussed above for Algorithm 1 does not hold for Algorithm 2.
\begin{algorithm}[H] 
\caption{Modified Thompson Sampling for Contextual bandits}
  \begin{algorithmic}
\FORALL{$t=1, 2,\ldots, $} 
		\STATE Sample $\tilde{\mu}(t)$ from distribution ${\cal N}(\hat{\mu}(t), \stdt^2 B(t)^{-1})$. \\
  	\STATE Play arm $a(t) := \arg \max_i b_i(t)^T\tilde{\mu}(t)$, and observe reward $r_{a(t)}(t)$.\\
\ENDFOR
  	\end{algorithmic}
\end{algorithm}


}

\nomenclature[mutilde]{$\tilde{\mu}(t)$}{$d$-dimensional sample generated by from distribution ${\cal N}(\hat{\mu}(t), \stdt^2 B(t)^{-1})$.}

\paragraph{Knowledge of time horizon $T$:} The parameter $\std = \stdvalueFixed$ can be replaced by $\std_t=\stdvalue$ at time $t$, if the time horizon $T$ is not known. In fact, this is the version of Thompson Sampling that we will analyze. The analysis we provide can be applied as it is (with only notational changes) to the version using the fixed value of $\std$ for all time steps, to get the same regret upper bound.

\paragraph{Computational efficiency:} Every step $t$ of Thompson Sampling (both algorithms) consists of generating a $d$-dimensional sample $\tilde{\mu}(t)$ from a multi-variate Gaussian distribution, and solving the problem $\arg \max_i b_i(t)^T\tilde{\mu}(t)$. Therefore, even if the number of arms $N$ is large (or infinite), the above algorithms are efficient as long as the problem $\arg \max_i b_i(t)^T\tilde{\mu}(t)$ is efficiently solvable. This is the case, for example, when the set of arms at time $t$ is given by a $d$-dimensional convex set ${\cal K}_t$ (every vector in ${\cal K}_t$ is a context vector, and thus corresponds to an arm). The problem to be solved at time step $t$ is then $\max_{b\in {\cal K}_t} b^T\tilde{\mu}(t)$, where ${\cal K}_t$.

\subsection{Our Results}


\begin{theorem}
\label{th:dep}	
With probability $1-\delta$, the total regret for Thompson Sampling algorithm in time $T$ is bounded as
\begin{equation}
\label{eq:bound1}
{\cal R}(T) = O\left(d^{3/2}\sqrt{T} \left( \ln (T) + \sqrt{\ln(T) \ln(\frac{1}{\delta})}\right)\right),
\end{equation}
or,
\begin{equation}
\label{eq:bound2}
{\cal R}(T) = O\left(d\sqrt{T \log(N)} \left( \ln (T) + \sqrt{\ln(T) \ln(\frac{1}{\delta})}\right)\right),
\end{equation}
whichever is smaller,
for any $0<\delta<1$, where $\delta$ is a parameter used by the algorithm.
\end{theorem}

\comment{

Then, the regret of Thompson Sampling Algorithm 1 and Algorithm 2 can be bounded by simply substituting $\epsilon_t=\epsilon$ and $\epsilon_t=\frac{1}{\ln t}$, respectively, in Theorem \ref{th:dep}.
\begin{corollary}
Thompson Sampling (Algorithm 1) achieves $\tilde{O}(d^{3/2}\sqrt{\frac{T^{(1+\epsilon)}}{\epsilon}})$ regret in time $T$. 
\end{corollary}

Note that if $T$ is known, one could choose $\epsilon = \frac{1}{\ln T}$, to get $\tilde{O}(d\sqrt{T} )$ regret bound for Thompson Sampling 	Algorithm 1. The next result shows that 
if we allow choosing a different $\epsilon$ at every time step $t$, i.e., $\epsilon_t = \frac{1}{\ln t}$, we can get  $\tilde{O}(d\sqrt{T} )$ regret bound without requiring the algorithm to know $T$. 

\begin{corollary}
Modified Thompson Sampling (Algorithm 2) with $\epsilon_t=\frac{1}{\ln t}$ achieves $\tilde{O}(d^{3/2}\sqrt{T})$ regret in time $T$. 
\end{corollary}
}
\begin{remark}
The regret bound in Equation \eqref{eq:bound1} does not depend on $N$, and are applicable to the case of infinite arms, with only notational changes required in the analysis.
\end{remark}

\subsection{Related Work}
\label{sec:related}
The contextual bandit problem with linear payoffs is a widely studied problem in 
statistics and machine learning often under different 
names as mentioned by \citet{DBLP:journals/jmlr/ChuLRS11}: bandit problems with co-variates \citep{Wood79, Sarkar91}, 
associative reinforcement learning \citep{DBLP:journals/ml/Kaelbling94}, associative bandit problems \citep{DBLP:journals/jmlr/Auer02, DBLP:conf/icml/StrehlMLH06}, bandit problems with expert advice \citep{AuerCFS02}, and linear bandits \citep{DBLP:conf/colt/DaniHK08, DBLP:conf/nips/Abbasi-YadkoriPS11, bubeck12}. The name \emph{contextual bandits} was coined in \citet{LangfordZ07}. 

A lower bound of $\Omega(d\sqrt{T})$ for this problem was given by \citet{DBLP:conf/colt/DaniHK08}, when the number of arms is allowed to be infinite. In particular, they prove their lower bound using an example where the set of arms correspond to all vectors in the intersection of a $d$-dimensional sphere and  a cube. They also provide an upper bound of $\tilde{O}(d\sqrt{T})$, although their setting is slightly restrictive in the sense that the context vector for every arm is fixed in advanced and is not allowed to change with time. \citet{DBLP:conf/nips/Abbasi-YadkoriPS11} analyze a UCB-style algorithm and provide a regret upper bound of $O(d \log{(T)} \sqrt{T} + \sqrt{dT\log{(T/\delta)}})$. 

For finite $N$,  \citet{DBLP:journals/jmlr/ChuLRS11} show a lower bound of $\Omega(\sqrt{Td})$ for $d^2 \leq T$. 
\citet{DBLP:journals/jmlr/Auer02} and \citet{DBLP:journals/jmlr/ChuLRS11} analyze SupLinUCB, a complicated algorithm using UCB as a subroutine, for this problem. \citet{DBLP:journals/jmlr/ChuLRS11} achieve a regret bound of $O(\sqrt{T d \ln^3(NT\ln(T)/\delta)})$ with probability at least $1-\delta$ (\citet{DBLP:journals/jmlr/Auer02} proves similar results). This regret bound is not applicable to the case of infinite arms, and assumes that context vectors are generated by an {\it oblivious} adversary. Also, this regret bound would give $O(d^2\sqrt{T})$ regret if $N$ is exponential in $d$. 
The state-of-the-art bounds for linear bandits problem in case of finite $N$ are given by \citet{bubeck12}.  They provide an algorithm based on exponential weights, with regret of order  $\sqrt{dT\log N}$ for any finite set of $N$ actions. This also gives $O(d\sqrt{T})$ regret when $N$ is exponential in $d$. 

However, none of the above algorithms is efficient when $N$ is large, in particular, when the arms are given by all points in a continuous set of dimension $d$. The algorithm of \citet{bubeck12} requires to maintain a distribution of $O(N)$ support, and those of \citet{DBLP:journals/jmlr/ChuLRS11}, \citet{DBLP:conf/colt/DaniHK08}, \cite{DBLP:conf/nips/Abbasi-YadkoriPS11} will need to solve an NP-hard problem at every step, even when the set of arms is given by a polytope of $d$-dimensions. In contrast, the Thompson Sampling algorithm we propose will run in time polynomial in $d$, as long as the one can efficiently optimize a linear function over the set of arms (maximize $b^T\tilde{\mu}(t)$ for $b\in {\cal K}$, where ${\cal K}$ is the set of arms). This can be done efficiently, for example, when the set of arms forms a convex set, and even for some combinatorial set of arms. We pay for this efficiency in terms of regret - our regret bounds are $\tilde{O}(d^{3/2}\sqrt{T})$ when $N$ is large or infinite, which is a factor of $\sqrt{d}$ away from the information theoretic lower bound. The only other efficient algorithm for this problem that we are aware of was provided by \citet{DBLP:conf/colt/DaniHK08} (Algorithm 3.2), which also achieves a regret bound of $O(d^{3/2} \sqrt{T})$. Thus, Thompson Sampling achieves the best regret upper bound achieved by an efficient algorithm in the literature. It is open problem to find a computationally efficient algorithm when $N$ is large or infinite, that achieves the information theoretic lower bound of $O(d\sqrt{T})$ on regret.

Our results demonstrate that the natural and efficient heuristic of Thompson Sampling can achieve theoretical bounds that are close to the best bounds. The main contribution of this paper is to provide new tools  for analysis of Thompson Sampling algorithm for contextual bandits, which despite being popular and empirically attractive, has eluded theoretical analysis. 
We believe the techniques used in this paper will provide useful insights into the workings of this Bayesian algorithm, and may be useful for further improvements and extensions.

\section{Regret Analysis: Proof of Theorem \ref{th:dep}}
\label{sec:proof}

\subsection{Challenges and proof outline} 
\label{sec:formal-outline}
The contextual version of the multi-armed bandit problem presents new challenges for the analysis of TS algorithm, and 
the techniques used so far for analyzing the basic multi-armed bandit problem by \citet{AgrawalG12, KaufmannMunos12} do not seem directly applicable. Let us describe some of these difficulties and our novel ideas to resolve them. 

In the basic MAB problem there are $N$ arms, with mean reward $\mu_i \in {\mathbb{R}}$ for arm $i$, and the regret for playing a suboptimal arm $i$ is $\mu_{a^*}-\mu_i$, where $a^*$ is the arm with the highest mean. Let us compare this to a $1$-dimensional contextual MAB problem, where arm $i$ is associated with a parameter $\mu_i \in {\mathbb{R}}$, but in addition, at every time $t$, it is associated with a context $b_i(t) \in {\mathbb{R}}$, so that mean reward is $b_i(t)\mu_i$. 
The best arm $a^*(t)$ at time $t$ is the arm with the highest mean at time $t$, and the regret for playing arm $i$ is $b_{a^*(t)}(t)\mu_{a^*(t)}-b_i(t)\mu_i$. 

In general, the basis of regret analysis for stochastic MAB is to prove that the variances of empirical estimates for all arms decrease fast enough, so that the regret incurred until the variances become small enough, is small. 
In the basic MAB, the variance of the empirical mean is inversely proportional to the number of plays $k_i(t)$ of arm $i$ at time $t$. Thus, every time the suboptimal arm $i$ is played, we know that even though a regret of $\mu_{i^*}-\mu_i \le 1$ is incurred, there is also an improvement of exactly $1$ in the number of plays of that arm, and hence, corresponding decrease in the variance.  The techniques for analyzing basic MAB rely on this observation to precisely quantify the exploration-exploitation tradeoff.
On the other hand, the variance of the empirical mean for the contextual case is given by inverse of $B_i(t) = \sum_{\tau=1: a(\tau)=i}^t b_i(\tau)^2$. When a suboptimal arm $i$ is played, if $b_i(t)$ is small, the regret $b_{a^*(t)}(t)\mu_{a^*(t)}-b_i(t)\mu_i$ could be much higher than the improvement $b_i(t)^2$ in $B_i(t)$. 

\addedShipra{In our proof, we overcome this difficulty by dividing the arms into two groups at any time: saturated and unsaturated arms, based on whether the standard deviation of the estimates for an arm is smaller or larger compared to the standard deviation for the optimal arm. The optimal arm is included in the group of unsaturated arms. We show that for the unsaturated arms, 
the regret on playing the arm can be bounded by a factor of the standard deviation, which improves every time the arm is played. This allows us to bound the total regret due to unsaturated arms. For the saturated arms, standard deviation is small, or in other words, the estimates of the means constructed so far 
are quite accurate in the direction of the current contexts of these arms, so that the algorithm is able to distinguish between them and the optimal arm. We utilize this observation to show that the probability of playing such arms is small, and at every time step an unsaturated arm will be played with some constant probability. 
\newline\\
Below is a more technical outline of the proof of Theorem \ref{th:dep}.}	 At any time step $t$, we divide the arms into two groups:
\begin{itemize}\setlength{\itemsep}{-1mm}
\item {\em \sat arms} defined as those with $\Delta_i(t) > \tdev \ s_i(t)$, 
\item {\em{\unsat arms}} defined as those with  $\Delta_i(t) \le \tdev \ s_i(t)$,
\end{itemize}
where $s_i(t)=\sqrt{b_i(t)^TB(t)^{-1}b_i(t)}$ and $\tdev$, $\cof$ ($\tdev > \cof$) are deterministic functions of $t, d, \delta$, defined later. 
Note that $s_i(t)$ is the standard deviation of the estimate $b_i(t)^T\hat{\mu}(t)$ and $\stdt s_i(t)$ is the standard deviation of the random variable $b_i(t)^T\tilde{\mu}(t)$. 
\removedShipra{Thus, intuitively, \sat arms are the arms with the property that the estimates of the means constructed so far in the direction of their current contexts are quite accurate, making the deviations $s_i(t)$s small enough---significantly smaller than their $\Delta_i(t)$.}

We use concentration bounds for $\tilde{\mu}(t)$ and $\hat{\mu}(t)$ to bound the regret at any time $t$ by $\tdev( s_{t,a^*(t)} + s_{a(t)}(t) )$. Now, if an \unsat arm is played at time $t$, then using the definition of \unsat arms, the regret is at most $\tdev s_{a(t)}(t)$. This is useful because of the inequality  $\sum_{t=1}^T s_{a(t)}(t) = O(\sqrt{Td\ln T})$ (derived along the lines of \citet{DBLP:journals/jmlr/Auer02}), which allows us to bound the total regret due to unsaturated arms.

To bound the regret irrespective of whether a \sat or \unsat arm is played at time $t$, 
we lower bound the probability of playing an \unsat arm at any time $t$.  More precisely, we define ${\cal F}_{t-1}$ as the union of history ${\cal H}_{t-1}$ and the contexts $b_i(t),i=1,\ldots, N$ at time $t$, and prove that for ``most" (in a high probability sense) ${\cal F}_{t-1}$, 
 \EQ{-0.05in}{-0.05in}{\PrR{\mbox{$a(t)$ is a \unsat arm} \condR {\cal F}_{t-1}} \ge \probt - \frac{1}{t^2},}
where $\probt=\p$. Note that for $\probt$ is constant for $\epsilon_t=1/\ln(t)$. 
This observation allows us to establish that the expected regret at any time step $t$ is upper bounded in terms of regret due to playing an \unsat arm at that time, i.e. in terms of $s_{a(t)}(t)$. More precisely, we prove that for ``most" ${\cal F}_{t-1}$
\EQ{0in}{0in}{\ExR{\reg(t) \condR {\cal F}_{t-1}} \le \frac{3\tdev}{\probt} \ExR{s_{a(t)}(t) \condR {\cal F}_{t-1}} + \frac{2\tdev}{\probt t^2}.}

We use these observations to 
establish that $(X_t; t\ge 0)$, where 
\EQ{0in}{0in}{X_t \simeq \reg(t) - \frac{3\tdev}{\probt} s_{a(t)}(t) - \frac{2\tdev}{\probt t^2},} 
is a super-martingale difference process adapted to filtration ${\cal F}_{t}$. Then, using the Azuma-Hoeffding inequality 
for super-martingales, along with the inequality $\sum_{t} s_{a(t)}(t) = O(\sqrt{Td\ln T})$, 
we will obtain the desired high probability regret bound. 


\subsection{Formal proof}
\label{subsec:proof}
As mentioned earlier, we will analyze the version of Algorithm 1 that uses $\stdt=\stdvalue$ instead of $\std=\stdvalueFixed$ at time $t$.  

We start with introducing some notations. For quick reference, the notations introduced below also appear in a table of notations at the beginning of the supplementary material.
\begin{definition}
For all $i$, define $\theta_i(t)=b_i(t)^T\tilde{\mu}(t)$, and $s_i(t)=\sqrt{b_i(t)^TB(t)^{-1}b_i(t)}$. 
By definition of $\tilde{\mu}(t)$ in Algorithm 2, marginal distribution of each $\theta_i(t)$ is Gaussian with mean $b_i(t)^T\hat{\mu}(t)$ and standard deviation $\stdt s_i(t)$. 
\nomenclature[theta]{$\theta_i(t)$}{$=b_i(t)^T\tilde{\mu}(t)$}
\nomenclature[s]{$s_i(t)$}{$=\sqrt{b_i(t)^TB(t)^{-1}b_i(t)}$}
\end{definition}
\begin{definition}
Recall that $\Delta_i(t)=b_{a^*(t)}(t)^T\mu - b_i(t)^T\mu$, the difference between the mean reward of optimal arm and arm $i$ at time $t$.
\nomenclature[Delta]{$\Delta_i(t)$}{$=b_{a^*(t)}(t)^T\mu - b_i(t)^T\mu$} 
\end{definition}

\begin{definition}
Define $\cof=\cofvalue$, 
\nomenclature[l]{$\cof$}{$=\cofvalue$}
$\stdt = \stdvalue$, \nomenclature[v_t]{$\stdt$}{$=\stdvalue$}
$\tdev= \logdev \stdt + \cof$, \nomenclature[g_t]{$\tdev$}{$=\logdev\stdt + \cof$}
and $\probt=\p$. \nomenclature[\probt]{$\probt$}{$=\p$}
\end{definition}

\begin{definition}
Define $\Emu(t)$ and $\Etheta(t)$ as the events that $b_i(t)^T\hat{\mu}(t)$ and $\theta_i(t)$ are concentrated around their respective means.
More precisely, define $\Emu(t)$ as the event that 
\EQ{-0.1in}{-0.05in}{\forall i: |b_i(t)^T\hat{\mu}(t) - b_i(t)^T\mu| \le \cof\  s_i(t).}

Define $\Etheta(t)$ as the event that
\EQ{-0.1in}{-0.05in}{\forall i: |\theta_i(t) - b_i(t)^T\hat{\mu}(t)| \le \logdev \stdt \ s_i(t).}

\end{definition}

\nomenclature[Emu]{$\Emu(t)$}{Event $\forall i:$ $\lvert b_i(t)^T\hat{\mu}(t)- b_i(t)^T\mu \rvert \le \cof s_i(t)$}
\nomenclature[Etheta]{$\Etheta(t)$}{Event $\forall i:$ $\lvert \theta_i(t) - b_i(t)^T\hat{\mu}(t) \rvert \le \logdev \stdt s_i(t)$}
\begin{definition}
An arm $i$ is called \emph{\sat} at time $t$ if $\Delta_i(t) > \tdev s_i(t)$, and \emph{\unsat} otherwise.
Let $C(t)$ denote the set of \sat arms at time $t$. 
Note that the optimal arm is always unsaturated at time $t$, i.e., $a^*(t) \notin C(t)$. An arm may keep shifting from \sat to \unsat and vice-versa over time.
\nomenclature[s]{saturated arm}{any arm $i$ with $\Delta_i(t) > \tdev s_i(t)$.}
\nomenclature[C(t)]{$C(t)$}{The set of saturated arms at time $t$.}
\end{definition}

\begin{definition}
Define filtration ${\cal F}_{t-1}$ as the union of history until time $t-1$, and the contexts at time $t$, i.e., ${\cal F}_{t-1}=\{{\cal H}_{t-1}, b_i(t), i=1,\ldots, N\}$. 
\nomenclature[F]{${\cal F}_{t-1}$}{$=\{{\cal H}_{t-1}, b_i(t), i=1,\ldots, N\}$}

\addedShipra{By definition, ${\cal F}_1\subseteq {\cal F}_2 \cdots \subseteq {\cal F}_{T-1}$. Observe that the following quantities are determined by the history ${\cal H}_{t-1}$ and the contexts $b_i(t)$ at time $t$, and hence are included in ${\cal F}_{t-1}$,
\vspace{-0.1in}
\begin{itemize}\setlength{\itemsep}{-1mm}
\item $\hat{\mu}(t), B(t)$,
\item $s_i(t)$, for all $i$,
\item the identity of the optimal arm $a^*(t)$ and the set of saturated arms $C(t)$,
\item whether $\Emu(t)$ is true or not,
\item the distribution ${\cal N}(\hat{\mu}(t), \stdt^2 B(t)^{-1})$ of $\tilde{\mu}(t)$, and hence the joint distribution of $\theta_i(t)=b_i(t)^T\tilde{\mu}(t), i=1,\ldots, N$.
\end{itemize}}
\end{definition}

\begin{lemma}
\label{lem:E(t)}
For all $t$, $0<\delta<1$, $ \Pr(\Emu(t)) \ge 1-\frac{\delta}{t^2}.$
And, for all possible filtrations ${\cal F}_{t-1}$, 
 $\Pr(\Etheta(t) | {\cal F}_{t-1}) \ge 1-\frac{1}{t^2}.$
\end{lemma}
\begin{proof}
The complete proof of this lemma appears in Appendix \ref{app:E(t)}. The probability bound for $\Emu(t)$ will be proven using a concentration inequality given by \citet{DBLP:conf/nips/Abbasi-YadkoriPS11}, stated as Lemma \ref{lem:dDim-concentration} in Appendix \ref{app:concentration}. The $R$-sub-Gaussian assumption on rewards will be utilized here. The probability bound for $\Etheta(t)$ will be proven using a concentration inequality for Gaussian random variables from  \citet{abramowitz+stegun} stated as Lemma \ref{lem:gauss} in Appendix \ref{app:concentration} . 
\end{proof}


The next lemma lower bounds the probability that $\theta_{a^*(t)}(t)=b_{a^*(t)}(t)^T\tilde{\mu}(t)$ for the optimal arm at time $t$ will exceed its mean reward $b_{a^*(t)}(t)^T\mu$.

\begin{lemma}
\label{lem:pBound}
For any filtration ${\cal F}_{t-1}$ such that $\Emu(t)$ is true, 
\EQ{0in}{0in}{\PrR{\theta_{a^*(t)}(t) > b_{a^*(t)}(t)^T\mu \condR  {\cal F}_{t-1}} \ge \probt.}
\end{lemma}
\begin{proof}
The proof uses anti-concentration of Gaussian random variable $\theta_{a^*(t)}(t)=b_{a^*(t)}(t)^T\tilde{\mu}(t)$, which has mean $b_{a^*(t)}(t)^T\hat{\mu}(t)$ and standard deviation $\stdt s_{t,a^*(t)}$, provided by Lemma \ref{lem:gauss} in Appendix \ref{app:concentration}, and the concentration of $b_{a^*(t)}(t)^T\hat{\mu}(t)$ around $b_{a^*(t)}(t)^T\mu$ provided by the event $\Emu(t)$. The details of the proof are in Appendix \ref{app:pBound}.
\end{proof}
The following lemma bounds the probability of playing saturated arms in terms of the probability of playing unsaturated arms.
\begin{lemma}
\label{lem:probOpt}
For any filtration ${\cal F}_{t-1}$ such that $\Emu(t)$ is true,
 $$\PrR{a(t) \notin C(t) \condR {\cal F}_{t-1}} \ge \probt -\frac{1}{t^2}.$$
\end{lemma}
\begin{proof}
The algorithm chooses the arm with the highest value of $\theta_i(t)=b_i(t)^T\tilde{\mu}(t)$ to be played at time $t$. 
Therefore, if $\theta_{a^*(t)}(t)$ is greater than $\theta_j(t)$ for all saturated arms, i.e., $\theta_{a^*(t)}(t)>\theta_j(t), \forall j\in C(t)$, then one of the unsaturated arms (which include the optimal arm and other suboptimal unsaturated arms) must be played. 
Therefore, 
\begin{eqnarray}
\label{eq:tmp1}
\concise{}{}{\PrR{a(t) \notin C(t) \condR {\cal F}_{t-1}}} \nonumber\\
\concise{}{}{\ge \PrR{\theta_{a^*(t)}(t) > \theta_j(t), \forall j\in C(t) \condR {\cal F}_{t-1}}}.
\end{eqnarray}
By definition, for all saturated arms, i.e. for all $j\in C(t)$, $\Delta_j(t) > \tdev s_{t,j}$. Also, if both the events $\Emu(t)$ and $\Etheta(t)$ are true then, by the definitions of these events, for all $j\in C(t)$, $\theta_j(t) \le b_j(t)^T\mu + \tdev s_{t,j} $.
Therefore, given an ${\cal F}_{t-1}$ such that $\Emu(t)$ is true, either $\Etheta(t)$ is false, or else for all $j\in C(t)$,
$$\theta_j(t) \le b_j(t)^T\mu + \tdev s_{t,j} \le b_{a^*(t)}(t)^T\mu.$$
Hence, for any ${\cal F}_{t-1}$ such that $\Emu(t)$ is true, 
\begin{eqnarray*}
& & \PrR{\theta_{a^*(t)}(t) > \theta_j(t), \forall j\in C(t) \condR  {\cal F}_{t-1}} \\
& \ge & \PrR{\theta_{a^*(t)}(t) >b_{a^*(t)}(t)^T\mu \condR  {\cal F}_{t-1}} \\
& & \hspace{0.1in} -\PrR{\overline{\Etheta(t)} \condR {\cal F}_{t-1}}\\
& \ge & \probt-\frac{1}{t^2}.
\end{eqnarray*}
The last inequality uses Lemma \ref{lem:pBound}  and Lemma \ref{lem:E(t)}.
\end{proof}

\begin{lemma}
\label{lem:ExpecPlayed}
For any filtration ${\cal F}_{t-1}$ such that $\Emu(t)$ is true,
$$\ExR{\Delta_{a(t)}(t) \condR {\cal F}_{t-1}} \le \frac{3\tdev}{\probt}  \ \ExR{s_{a(t)}(t) \condR {\cal F}_{t-1}} + \frac{2 \tdev}{\probt t^2}.$$
\end{lemma}
\begin{proof}
Let $\bar{a}(t)$ denote the \unsat arm with smallest $s_i(t)$, i.e.
$$\bar{a}(t) = \arg \min_{i\notin C(t)} s_i(t)$$
Note that since $C(t)$ and $s_i(t)$ for all $i$ are fixed on fixing ${\cal F}_{t-1}$, so is $\bar{a}(t)$. 

Now, using Lemma \ref{lem:probOpt}, for any ${\cal F}_{t-1}$ such that $\Emu(\theta)$ is true,

\begin{eqnarray*}
\ExR{s_{a(t)}(t) \condR {\cal F}_{t-1}} & \ge & \ExR{s_{a(t)}(t) \condR {\cal F}_{t-1}, a(t)\notin C(t)} \\
& & \cdot \PrR{a(t)\notin C(t) \condR {\cal F}_{T-1}}\\
& \ge & s_{t, \bar{a}(t)} \left(\probt-\frac{1}{t^2}\right).
\end{eqnarray*}

Now, if events $\Emu(t)$ and $\Etheta(t)$ are true, then for all $i$, by definition, $\theta_i(t) \le b_i(t)^T\mu+\tdev s_i(t)$. Using this observation along with the fact that $\theta_{a(t)}(t) \ge\theta_i(t)$ for all $i$,
\begin{eqnarray*}
\Delta_{a(t)}(t) & = & \Delta_{\bar{a}(t)}(t) + (b_{\bar{a}(t)}(t)^T\mu - b_{a(t)}(t)^T\mu)\\
& \le &  \Delta_{\bar{a}(t)}(t) + (\theta_{\bar{a}(t)}(t)-\theta_{a(t)}(t)) \\
& & + \tdev s_{t,\bar{a}(t)} + \tdev s_{a(t)}(t)\\
& \le & \Delta_{\bar{a}(t)}(t) + \tdev s_{t,\bar{a}(t)} + \tdev s_{a(t)}(t)\\
& \le & \tdev s_{t, \bar{a}}(t) +  \tdev s_{t,\bar{a}(t)} + \tdev s_{a(t)}(t)
\end{eqnarray*}

Therefore, for any ${\cal F}_{t-1}$ such that $\Emu(\theta)$ is true either $\Delta_{a(t)}(t) \le 2\tdev s_{t, \bar{a}}(t) + \tdev s_{a(t)}(t)$ or $\Etheta(t)$ is false. Therefore, 
\begin{eqnarray*}
\ExR{\Delta_{a(t)}(t) \condR {\cal F}_{t-1}} & \le &  \ExR{2\ \tdev s_{t, \bar{a}}(t) + \tdev s_{a(t)}(t) \condR {\cal F}_{t-1}} \\
& & + \PrR{\overline{\Etheta(t)}}\\
& \le & \frac{2\ \tdev}{\left(\probt-\frac{1}{t^2}\right)} \ \ExR{s_{t, a(t)} \condR {\cal F}_{t-1}} \\
& & + \tdev \ExR{ s_{t, a(t)} \condR {\cal F}_{t-1}} + \frac{1}{t^2}\\
& \le &  \frac{3}{\probt}\ \tdev \ \ExR{s_{t, a(t)} \condR {\cal F}_{t-1}} + \frac{2 \tdev}{\probt t^2}.
\end{eqnarray*}
In the first inequality we used that for all $i$, $\Delta_i(t) \le 1$. The second inequality used the inequality derived in the beginning of this proof, and Lemma \ref{lem:E(t)} to apply $\PrR{\overline{\Etheta(t)}} \le \frac{1}{t^2}$. The third inequality used the observation that $0 \le s_{t, a(t)} \le ||b_{a(t)}(t)|| \le 1$.

\end{proof}
\begin{definition}
Recall that $\reg(t)$ was defined as, $ \reg(t)= \Delta_{a(t)}(t)=b_{a^*(t)}(t)^T\mu - b_{a(t)}(t)^T\mu.$
Define $ \reg'(t) = \reg(t) \cdot I(\Emu(t)).$
\end{definition}
Next, we establish a super-martingale process that will form the basis of our proof of the high-probability regret bound.
\begin{definition} 
\label{def:martingale}
Let 
\begin{eqnarray*}
\concise{X_t}{ = }{\reg'(t) - \frac{3\tdev}{\probt} s_{a(t)}(t)} - \frac{2\tdev}{\probt t^2} \\
\concise{Y_t}{ = }{ \sum_{w=1}^t X_w.}
\end{eqnarray*}
\end{definition}

\begin{lemma}
\label{lem:main}
$(Y_t; t = 0,\ldots, T)$ is a super-martingale process with respect to filtration ${\cal F}_t$.
\end{lemma}
\begin{proof}
See Definition \ref{def:supermartingale} in Appendix \ref{app:concentration} for the definition of super-martingales. We need to prove that for all $t \in [1,T]$, and any ${\cal F}_{t-1}$, $\Ex[Y_t-Y_{t-1} | {\cal F}_{t-1}] \le 0$, i.e.

$$\ExR{\reg'(t) \condR {\cal F}_{t-1}} \le \frac{3\tdev}{\probt} \ExR{ s_{a(t)}(t) \condR {\cal F}_{t-1}} + \frac{2\tdev}{\probt t^2} .$$

Note that whether $\Emu(t)$ is true or not is completely determined by ${\cal F}_{t-1}$. 
If ${\cal F}_{t-1}$ is such that $\Emu(t)$ is not true, then $\reg'(t)=\reg(t)\cdot I(\Emu(t))=0$, and the above inequality holds trivially. And, for ${\cal F}_{t-1}$ such that $\Emu(t)$ holds, the inequality follows from Lemma \ref{lem:ExpecPlayed}.
\end{proof}
Now, we are ready to prove Theorem \ref{th:dep}. 
\paragraph{Proof of Theorem \ref{th:dep}}

Note that $X_t$ is bounded, $|X_t| \le 1+\frac{3}{\probt} \tdev+\frac{2}{\probt t^2} \tdev \le \frac{6}{\probt} \tdev$.
Thus, we can apply Azuma-Hoeffding inequality (see Lemma \ref{lem:azuma} in Appendix \ref{app:concentration}), to obtain that with probability $1-\frac{\delta}{2}$,\vspace{-0.1in}

\begin{eqnarray}
\label{eq:azumaApp}
\concise{\sum_{t=1}^T  \reg'(t)}{\le }{\sum_{t=1}^T \frac{3\tdev}{\probt}  s_{a(t)}(t) + \sum_{t=1}^T \frac{2\tdev}{\probt t^2} + \sqrt{2\left(\sum_t \frac{36 \tdev^2}{\probt^2}\right)\ln(\frac{2}{\delta})}}\nonumber\\
\end{eqnarray}

Note that $\probt$ is a constant. 
Also, by definition, $\tdev \le \tdevT$. Therefore, from above equation, with probability $1-\frac{\delta}{2}$,
\begin{eqnarray*}
\concise{\sum_{t=1}^T  \reg'(t)}{\le }{\frac{3\tdevT}{\probT}  \sum_{t=1}^T s_{a(t)}(t) +  \frac{2\tdevT}{\probT} \sum_{t=1}^T \frac{1}{t^2}+ \frac{6 \tdevT}{\probT}\sqrt{2T\ln(\frac{2}{\delta})}}\\
\end{eqnarray*}

Now, we can use $\sum_{t=1}^T s_{a(t)}(t) \le 5\sqrt{dT\ln T},$ which can be derived along the lines of Lemma 3 of \citet{DBLP:journals/jmlr/ChuLRS11} using Lemma 11 of \citet{DBLP:journals/jmlr/Auer02} (see Appendix \ref{app:proof} for details).
Also, by definition 
$\tdevT=O(\sqrt{d\ln(\frac{T}{\delta})} \cdot (\min\{\sqrt{d}, \sqrt{\log(N)}\}))$ (see the Table of notations in the beginning of the supplementary material). Substituting in above, we get 
\begin{eqnarray*}
\concise{\sum_{t=1}^T  \reg'(t)}{ = }{ O\left(d\sqrt{ \ln(\frac{T}{\delta})}  \cdot (\min\{\sqrt{d}, \sqrt{\log(N)}\})) \cdot \sqrt{dT\ln T}\right)}\\
\concise{}{=}{O\left(d\sqrt{T}  \cdot (\min\{\sqrt{d}, \sqrt{\log(N)}\})) \cdot \left( \ln(T) +\sqrt{\ln(T)\ln(\frac{1}{\delta})}\right)\right)}. \vspace{-0.1in}
\end{eqnarray*}

Also, because $\Emu(t)$ holds for all $t$ with probability at least $1-\frac{\delta}{2}$ (see Lemma \ref{lem:E(t)}), $\reg'(t)=\reg(t)$ for all $t$ with probability at least $1-\frac{\delta}{2}$. Hence, with probability $1-\delta$,
\begin{eqnarray*}
\concise{{\cal R}(T)}{ = }{\sum_{t=1}^T \reg(t) =  \sum_{t=1}^T \reg'(t)}\\
\concise{}{= }{O\left(d\sqrt{T} \cdot  (\min\{\sqrt{d}, \sqrt{\log(N)}\}))\cdot \left( \ln(T) +\sqrt{\ln(T)\ln(\frac{1}{\delta})}\right)  \right).}
\end{eqnarray*}
\comment{
\paragraph{Proof of Theorem \ref{th:dep}}
Here, we use $\epsilon_t=\frac{1}{\ln t}$, so that $\probt=\p = \pFixed =: p$. Substituiting $\probt=p$, and $\tdev \le \tdevT$, in Equation \ref{eq:azumaApp}, we get with probability $1-\frac{delta}{2}$,
\begin{eqnarray*}
\concise{\sum_{t=1}^T  \reg'(t)}{\le }{\frac{3\tdevT}{p}  \sum_{t=1}^T s_{a(t)}(t) +  \frac{2\tdevT}{p } \sum_{t=1}^T \frac{1}{t^2}+ \frac{6 \tdevT}{p}\sqrt{2T\ln(\frac{2}{\delta})}}\\
\end{eqnarray*}
Again, using $\sum_{t=1}^T s_{a(t)}(t) \le 5\sqrt{dT\ln T},$ and $\tdevT=O(\sqrt{\frac{d}{\epsilon} \ln(T)\ln(\frac{1}{\delta}})$, with probability $1-\frac{\delta}{2}$,
\begin{eqnarray*}
\concise{\sum_{t=1}^T  \reg'(t)}{ = }{ O\left(\sqrt{d \ln(T) \ln(\frac{1}{\delta})} \cdot \sqrt{dT\ln T}\right)}\\
\concise{}{=}{O\left(d\sqrt{T} \left(\ln(T) \sqrt{\ln(\frac{1}{\delta})}\right)\right)}. \vspace{-0.1in}
\end{eqnarray*}
And, with probability $1-\delta$,
\EQ{0in}{0in}{ {\cal R}(T) =O\left(d\sqrt{T} \left(\ln(T) \sqrt{\ln(\frac{1}{\delta})}\right)\right).}
}

The proof for the alternate definition of regret mentioned in Remark \ref{rem:regret} is provided in Appendix \ref{app:proof}.


\section{Conclusions}
\label{sec:conclusions}
We provided a theoretical analysis of Thompson Sampling for the stochastic contextual bandits problem with linear payoffs. Our results resolve some open questions regarding the theoretical guarantees for Thompson Sampling, and establish that even for the contextual version of the stochastic MAB problem, TS achieves regret bounds close to the state-of-the-art methods. We used a novel martingale-based analysis technique which is arguably simpler than the techniques in the past work on TS \citep{AgrawalG12, KaufmannMunos12}, and is amenable to extensions. 

In the algorithm in this paper, Gaussian priors were used, so that $\tilde{\mu}(t)$ was generated from a Gaussian distribution. However, the analysis techniques in this paper are extendable to an algorithm that uses a prior distribution other 
than the Gaussian distribution. The only distribution specific properties we have used in the analysis are the concentration and anti-concentration inequalities for Gaussian distributed random variables (Lemma \ref{lem:gauss}), \addedShipra{which were used to prove Lemma \ref{lem:E(t)} and Lemma \ref{lem:pBound} respectively.}
\removedShipra{The concentration inequality was used to prove that $\Etheta(t)$ happens with high probability in Lemma \ref{lem:E(t)}, and the anti-concentration inequality was used to lower bound the probability that Gaussian distributed random variable $\theta_{i*(t)}(t)$ exceeds its mean by some factors of its standard deviation  in  Lemma \ref{lem:pBound}.} 
If any other distribution provides similar tail inequalities, to allow us proving these lemmas, these can be used as a black box in the analysis, and the regret bounds can be reproduced for that distribution.

Several questions remain open. A tighter analysis that can remove the dependence on $\epsilon$ is desirable. We believe that our techniques would adapt to provide such bounds for the {\it expected regret}. Other avenues to explore are contextual bandits with {\em generalized} linear models considered in \citet{DBLP:conf/nips/FilippiCGS10}, the setting with delayed and batched feedback, and the {\em agnostic} case of contextual bandits with linear payoffs. The agnostic case refers to the setting which does not make the realizability assumption that there exists a vector $\mu_i$ for each $i$ for which $\Ex[r_i(t) | b_i(t)]=b_i(t)^T\mu_i$. To our knowledge, no existing algorithm has been shown to have non-trivial regret bounds for the agnostic case. 

\paragraph{Acknowledgements} This is an extended and slightly modified version of \cite{AgrawalGoyalICML13}.

\bibliography{bibliography_contextual}
\bibliographystyle{icml2013}

\newpage

\printnomenclature
\newpage

\appendix

\section{}
\label{app:th:dep}

\subsection{Posterior distribution computation}
\label{app:posterior}
\begin{eqnarray*}
	& & \Pr(\tilde{\mu} | r_i(t)) \\
	& \propto & \Pr(r_i(t) | \tilde{\mu}) \Pr(\tilde{\mu})\\
	& \propto & \exp\{ -\frac{1}{2\std^2}   ( (r_i(t) - \tilde{\mu}^Tb_i(t))^2 \\
	& &  \hspace{0.3in}+ (\tilde{\mu} - \hat{\mu}(t))^TB(t)(\tilde{\mu} - \hat{\mu}(t))\}\\
	& \propto & \exp\{ -\frac{1}{2 \std^2}   ( r_i(t)^2 + \tilde{\mu}^T b_i(t)b_i(t)^T\tilde{\mu} \\
	& & \hspace{0.3in} + \tilde{\mu}^TB(t)\tilde{\mu} - 2\tilde{\mu}^Tb_i(t)r_i(t)-2\tilde{\mu}^TB(t)\hat{\mu}(t))\}\\
	& \propto & \exp\{ -\frac{1}{2\std^2}   ( \tilde{\mu}^TB(t+1)\tilde{\mu} - 2\tilde{\mu}^T  B(t+1)\hat{\mu}(t+1))\}\\
	& \propto & \exp\{ -\frac{1}{2\std^2}   ( \tilde{\mu} - \hat{\mu}(t+1))^T B(t+1) ( \tilde{\mu} - \hat{\mu}(t+1))\}\\
	& \propto & {\cal N}(\hat{\mu}(t+1), \std^2 {B(t+1)}^{-1}).
	\end{eqnarray*}
	Therefore, the posterior distribution of $\mu$ at time $t+1$ is ${\cal N}(\hat{\mu}(t+1), \std^2 {B(t+1)}^{-1})$. 

\subsection{Some concentration inequalities}
\label{app:concentration}
Formula 7.1.13 from \citet{abramowitz+stegun} can be used to derive the following concentration and anti-concentration inequalities for Gaussian distributed random variables.
\begin{lemma}
\label{lem:gauss}\citep{abramowitz+stegun}
For a Gaussian distributed random variable $Z$ with mean $m$ and variance $\sigma^2$, for any $z \ge 1$,
$$ \frac{1}{2\sqrt{\pi}z} e^{-z^2/2} \le \Pr( |Z-m| > z\sigma) \le \frac{1}{\sqrt{\pi}z} e^{-z^2/2}.$$
\end{lemma}

\begin{definition}[Super-martingale]
\label{def:supermartingale}
A sequence of random variables $(Y_t; t\ge 0)$ is called a super-martingale corresponding to filtration ${\cal F}_t$, if for all $t$, $Y_t$ is ${\cal F}_t$-measurable, and for $t\ge 1$,
$$\ExR{Y_t - Y_{t-1} \condR {\cal F}_{t-1}} \le 0.$$
\end{definition}
\begin{lemma}[Azuma-Hoeffding inequality]
\label{lem:azuma}
If a super-martingale $(Y_t; t\ge 0)$, corresponding to filtration ${\cal F}_t$, satisfies $|Y_t-Y_{t-1}| \le c_t$ for some constant $c_t$, for all $t=1, \ldots, T$, then for any $a\ge 0$,
	$$\Pr(Y_T-Y_0 \ge a) \le e^{-\frac{a^2}{2\sum_{t=1}^T c_t^2}}.$$
\end{lemma}

\comment{
\begin{lemma}
\label{lem:not-martingale}
If a sequence of random variables $(Y_t; t\ge 0)$ is such that $Y_t$ is ${\cal F}_t$-measurable, and for all $t\ge 1$,
$$ \forall \lambda\in {\mathbb{R}}, \ExR{e^{\lambda(Y_t-Y_{t-1})} \condR {\cal F}_{t-1}} \le e^{\lambda^2c_t^2/2},$$
	where $(c_1, c_2, \ldots, )$ are positive constants, then ,
	$$\Pr(|Y_T-Y_0| \ge r) \le e^{-\frac{r^2}{2\sum_{t=1}^T c_t^2}}$$
\end{lemma}
}

The following lemma is implied by Theorem 1 in \citet{DBLP:conf/nips/Abbasi-YadkoriPS11}:
\begin{lemma}\citep{DBLP:conf/nips/Abbasi-YadkoriPS11}
\label{lem:dDim-concentration}
Let $({\cal F}'_t; t\ge 0)$ be a filtration, $(m_t; t\ge 1)$ be an $\mathbb{R}^d$-valued stochastic process such that $m_t$ is $({\cal F}'_{t-1})$-measurable, $(\eta_t; t\ge 1)$ be a real-valued martingale difference process such that $\eta_t$ is $({\cal F}'_t)$-measurable.
For $t\ge 0$, define $\xi_t =\sum_{\tau=1}^t m_{\tau}\eta_{\tau}$ and $M_t=I_d+\sum_{\tau=1}^t m_{\tau}m_{\tau}^T$, where $I_d$ is the $d$-dimensional identity matrix. Assume $\eta_t$ is conditionally $R$-sub-Gaussian. 

Then, for any $\delta'>0$, $t\ge 0$, with probability at least $1-\delta'$,
$$ ||\xi_t||_{M_t^{-1}} \le R\sqrt{d\ln\left(\frac{t+1}{\delta'}\right)},$$
where $||\xi_t||_{M_t^{-1}} = \sqrt{\xi_t^T M_t^{-1}\xi_t}$.
\end{lemma}

\subsection{Proof of Lemma \ref{lem:E(t)}}
\label{app:E(t)}
\paragraph{Bounding the probability of event $\Emu(t)$:}
We use Lemma \ref{lem:dDim-concentration} with $m_t=b_{a(t)}(t)$, $\eta_t=r_{a(t)}(t)-b_{a(t)}(t)^T\mu$, ${\cal F}'_{t} = (a(\tau+1), m_{\tau+1}, \eta_{\tau} : \tau\le t)$.  (Note that effectively, ${\cal F}'_{t}$ has all the information, including the arms played, until time $t+1$, except for the reward of the arm played at time $t+1$). By the definition of  ${\cal F}'_t$, $m_t$ is ${\cal F}'_{t-1}$-measurable, and $\eta_t$ is ${\cal F}'_t$-measurable. Also, $\eta_t$ \addedShipra{is conditionally $R$-sub-Gaussian due to the assumption mentioned in the problem settings (refer to Section \ref{sec:setting}), and} is a martingale difference process:
$$ \ExR{\eta_t | {\cal F}'_{t-1}} = \Ex[r_{a(t)}(t) | b_{a(t)}(t), a(t)] - b_{a(t)}(t)^T\mu = 0.$$
Also, this makes 
$$M_t = I_d +\sum_{\tau=1}^t m_{\tau}m_{\tau}^T = I_d + \sum_{\tau=1}^t b_{a(\tau)}(\tau)  b_{a(\tau)}(\tau)^T,$$
$$\xi_t = \sum_{\tau=1}^t m_{\tau}\eta_{\tau} = \sum_{\tau=1}^t b_{a(\tau)}(\tau) (r_{a(\tau)}-b_{a(\tau)}(\tau)^T\mu).$$
Note that $B(t) = M_{t-1}$, and $\hat{\mu}(t)-\mu = M_{t-1}^{-1}(\xi_{t-1}-\mu)$. Let for any vector $y \in \mathbb{R}$ and matrix $A\in \mathbb{R}^{d\times d}$, $||y||_A$ denote $\sqrt{y^TAy}$. Then, for all $i$,
\EQ{0in}{0in}{| b_i(t)^T\hat{\mu}(t) - b_i(t)^T\mu| = |b_i(t)^T M_{t-1}^{-1} (\xi_{t-1}-\mu)| \le ||b_i(t)||_{M_{t-1}^{-1}} ||\xi_{t-1}-\mu||_{M_{t-1}^{-1}} = ||b_i(t)||_{B(t)^{-1}} ||\xi_{t-1}-\mu||_{M_{t-1}^{-1}}.}
The inequality holds because $M_{t-1}^{-1}$ is a positive definite matrix.
Using Lemma \ref{lem:dDim-concentration}, for any $\delta'>0$, $t\ge 1$, with probability at least $1-\delta'$,
$$ ||\xi_{t-1}||_{M_{t-1}^{-1}} \le R\sqrt{d\ln\left(\frac{t}{\delta'}\right)}.$$
Therefore, $||\xi_{t-1}-\mu||_{M_{t-1}^{-1}} \le R\sqrt{d\ln\left(\frac{t}{\delta'}\right)}+||\mu||_{M_{t-1}^{-1}} \le R\sqrt{d\ln\left(\frac{t}{\delta'}\right)}+1$.
Substituting $\delta'=\frac{\delta}{t^2}$, we get that with probability $1-\frac{\delta}{t^2}$, for all $i$,
\begin{eqnarray*}
& & | b_i(t)^T\hat{\mu}(t) - b_i(t)^T\mu| \\
& \le & ||b_i(t)||_{B(t)^{-1}}\cdot\left(R\sqrt{d\ln\left(\frac{t}{\delta'}\right)}+1\right) \\
& \le & ||b_i(t)||_{B(t)^{-1}} \cdot \left(\cofvalue\right) \\
& = & \cof\ s_i(t).
\end{eqnarray*}

This proves the bound on the probability of $\Emu(t)$. 
\paragraph{Bounding the probability of event $\Etheta(t)$:}
 Given any filtration ${\cal F}_{t-1}$, $b_i(t), B(t)$ are fixed. Then,
\begin{eqnarray*}
|\theta_i(t)-b_i(t)^T\hat{\mu}(t)| & = & |b_i(t)^T(\tilde{\mu}(t)-\hat{\mu}(t))| \\
& = &  |b_i(t)^TB(t)^{-1/2}B(t)^{1/2}(\tilde{\mu}(t)-\hat{\mu}(t))|\\
& \le & \stdt  \sqrt{b_i(t)^TB(t)^{-1}b_i(t)} \cdot \left\|\left(\frac{1}{\stdt}B(t)^{1/2}(\tilde{\mu}(t)-\hat{\mu}(t))\right)\right\|_2\\
& = & \stdt s_i(t) ||\zeta||_2 \\
& \le & \stdt s_i(t)\sqrt{4d \ln t} \\
\end{eqnarray*}
with probability $1-\frac{1}{t^2}$. Here, $\zeta_k, k=1,\ldots, d$ denotes standard univariate 
normal random variable (mean $0$ and variance $1$). 

Alternatively, we can bound $|\theta_i(t)-b_i(t)^T\hat{\mu}(t)|$ for every $i$ by considering that $\theta_i(t)$ Gaussian random variable with mean $b_i(t)^T\hat{\mu}(t)$ and variance $\stdt^2 s_i(t)^2$. Therefore, using Lemma \ref{lem:gauss}, for every $i$
\begin{eqnarray*}
|\theta_i(t)-b_i(t)^T\hat{\mu}(t)| & = & \sqrt{4\ln(Nt)}s_i(t)
\end{eqnarray*}
with probability $1-\frac{1}{Nt^2}$. Taking union bound over $i=1,\ldots, N$, we obtain that $|\theta_i(t)-b_i(t)^T\hat{\mu}(t)| \le \sqrt{4\ln(Nt)}s_i(t)$ holds for all arms with probability $1-\frac{1}{t^2}$. 

Combined, the two bounds give that $\Etheta(t)$ holds with probability $1-\frac{1}{t^2}$. 
\subsection{Proof of Lemma \ref{lem:pBound}}
\label{app:pBound}
Given event $\Emu(t)$, $|b_{a^*(t)}(t)^T\hat{\mu}(t) - b_{a^*(t)}(t)^T\mu| \le \cof s_{a^*(t)}(t)$. 
And, since 
Gaussian random variable $\theta_{a^*(t)}(t)$ has mean $b_{a^*(t)}(t)^T\hat{\mu}(t)$ and standard deviation $\stdt s_{a^*(t)}(t)$,  using anti-concentration inequality in Lemma \ref{lem:gauss}, 
\begin{eqnarray*}
\concise{}{}{\PrR{\theta_{a^*(t)}(t) \ge b_{a^*(t)}(t)^T\mu \condR  {\cal F}_{t-1}}} \\
\concise{}{=}{ \Pr\left( \frac{\theta_{a^*(t)}(t)- b_{a^*(t)}(t)^T\hat{\mu}(t)}{\stdt s_{t,a^*(t)}} \right.}\\
\concise{}{}{\hspace{0.4in} \ge    \left.\frac{b_{a^*(t)}(t)^T  \mu - b_{a^*(t)}(t)^T\hat{\mu}(t)}{\stdt s_{t,a^*(t)}} \condR  {\cal F}_{t-1}\right)}\\
\concise{}{\ge}{\frac{1}{4\sqrt{\pi}} e^{-Z_t^2}}.
\end{eqnarray*}
where 
\begin{eqnarray*}
|Z_t| & = &  \left|\frac{b_{a^*(t)}(t)^T  \mu - b_{a^*(t)}(t)^T\hat{\mu}(t)}{\stdt s_{a^*(t)}(t)}\right|\\
&  \le & \   \frac{\cof s_{a^*(t)}(t)}{\stdt s_{a^*(t)}(t)} \\
& = & \frac{\left(\cofvalue \right)}{\stdvalue}\\
& \le & 1.
\end{eqnarray*}
So
\EQ{0in}{0in}{\PrR{\theta_{a^*(t)}(t) \ge b_{a^*(t)}(t)^T\mu \condR  {\cal F}_{t-1}} \ge \frac{1}{4e\sqrt{\pi}}.}

\subsection{Missing details from Section \ref{subsec:proof}}
\label{app:proof}
\comment{
Note that for super-martingale $Y_t$,
\EQ{-0in}{-0in}{|Y_t-Y_{t-1}| = |X_t| = |\frac{1}{\tdev} \reg'(t) - \frac{1}{p} \impr s_{t,a^*(t)} - s_{a(t)}(t) - \frac{2}{ pT^2}| \le \frac{5}{p}.}
The last inequality holds because for any $i$, $s_i(t) = \sqrt{b_i(t)^TB(t)^{-1}b_i(t)} \le ||b_i(t)||_2 \le 1$, $|\reg'(t)| \le 1$.
Therefore, by Azuma-Hoeffding inequality (Lemma \ref{lem:azuma}),
\EQ{-0in}{-0in}{\PrR{Y_T-Y_0 > \frac{5}{p}\sqrt{2T \ln(2/\delta)}} \le \exp\left( \frac{-\ln(2/\delta)T}{T}\right) \le \delta/2.}
Therefore with probability $1-\frac{\delta}{2}$,
\begin{eqnarray*}
\concise{}{}{\sum_{t=1}^T \frac{1}{\tdev} \reg'(t)} \\
\concise{}{\le}{\sum_{t=1}^T \left(\frac{1}{p} \impr s_{t,a^*(t)} + s_{a(t)}(t) + \frac{2}{ pT^2} \right)}\\
\concise{}{}{\hspace{0.1in}+\frac{5}{p}\sqrt{2T\ln(2/\delta)}}\\
\concise{}{ \le }{\frac{1}{p} \sum_{t=1}^T s_{a(t)}(t) + \sum_{t=1}^T  s_{a(t)}(t) + \frac{2}{pT} +\frac{5}{p}\sqrt{2T\ln(2/\delta)}}\\
\concise{}{= }{O(\sqrt{T^{1+\epsilon}d\ln T} + \sqrt{T^{1+\epsilon}\ln(1/\delta)})}
\end{eqnarray*}
For the last inequality, we use that $\sum_{t=1}^T s_{a(t)}(t) \le 5\sqrt{dT\ln T},$ which can be derived along the lines of Lemma 3 of \citet{DBLP:journals/jmlr/ChuLRS11} using Lemma 11 of \citet{DBLP:journals/jmlr/Auer02}. 
}
To derive the  inequality $\sum_{t=1}^T s_{a(t)}(t) \le 5\sqrt{dT\ln T}$, we use the following result, implied by the referred lemma in \citet{DBLP:journals/jmlr/Auer02}. 
\begin{lemma}{\citep[Lemma 11]{DBLP:journals/jmlr/Auer02}}.
\label{lem:Auer}
Let $A'=A + xx^T$, where $x\in {\mathbb R}^d, A, A'\in {\mathbb R}^{d\times d}$, and all the eigenvalues $\lambda_j,j=1,\ldots, d$ of $A$ are greater than or equal to $1$. Then, the eigenvalues $\lambda'_j, j=1,\ldots, d$ of $A'$ can be arranged so that $\lambda_j \le \lambda'_{j}$ for all $j$, and
$$ x^TA^{-1}x \le 10 \sum_{j=1}^d \frac{\lambda'_j-\lambda_j}{\lambda_j}.$$
\end{lemma}
Let $\lambda_{j,t}$ denote the eigenvalues of $B(t)$. Note that $B(t+1) = B(t) + b_{a(t)}(t)b_{a(t)}(t)^T$, and $\lambda_{j,t}\ge 1, \forall j$. Therefore, above implies 
$$s_{a(t)}(t)^2 \le 10 \sum_{j=1}^d \frac{\lambda_{j,t+1}-\lambda_{j,t}}{\lambda_{j,t}}.$$
This allows us to derive the given inequality after some algebraic computations following along the lines of Lemma 3 of \citet{DBLP:journals/jmlr/ChuLRS11}.

\comment{
Therefore, with probability $1-\frac{\delta}{2}$,
\EQ{-0.05in}{-0.05in}{\sum_{t=1}^T \reg'(t) \le 
O\left(d\sqrt{\frac{T^{1+\epsilon}}{\epsilon}} \left(\sqrt{\ln N \ln \frac{T}{\delta}} + \ln T\sqrt{\ln \frac{1}{\delta}}  + \ln \frac{1}{\delta}\sqrt{\ln T }\right) \right)}
Also, because $\Emu(t)$ holds for all $t$ with probability at least $1-\frac{\delta}{2}$ (Lemma \ref{lem:E(t)}), $\reg'(t)=\reg(t)$ for all $t$ with probability at least $1-\frac{\delta}{2}$. Hence, with probability $1-\delta$,
\EQ{-0.05in}{-0.05in}{ {\cal R}(T) = \sum_{t=1}^T \reg(t) = \sum_{t=1}^T \reg'(t) = O\left(d\sqrt{\frac{T^{1+\epsilon}}{\epsilon}} \left(\sqrt{\ln N \ln \frac{T}{\delta}} + \ln T\sqrt{\ln \frac{1}{\delta}}  + \ln \frac{1}{\delta}\sqrt{\ln T }\right) \right).}
}

To obtain bounds for the other definition of regret in Remark \ref{rem:regret}, we observe that because $\Ex[r_i(t) | {\cal F}_{t-1} ] = b_i(t)^T\mu$ for all $i$, the expected value of $\reg'(t)$ given ${\cal F}_{t-1}$ for this definition of $\reg(t)$ is same as before. More precisely, for ${\cal F}_{t-1}$ such that $\Emu(t)$ holds,
\begin{eqnarray*}
& & \ExR{\reg'(t) \condR {\cal F}_{t-1}} \\
&  = & \ExR{\reg(t) \condR {\cal F}_{t-1}}\\
& =& \ExR{r_{a^*(t)}(t) - r_{a(t)}(t) \condR {\cal F}_{t-1}} \\
& = & \ExR{b_{a^*(t)}(t)^T\mu-b_{a(t)}(t)^T\mu \condR {\cal F}_{t-1}}. 
\end{eqnarray*}
And, $\ExR{\reg'(t) \condR {\cal F}_{t-1}} =0$ for other ${\cal F}_{t-1}$.
Therefore, Lemma \ref{lem:main} holds as it is, and $Y_t$ defined in Definition \ref{def:martingale} is a super-martingale with respect to this new definition of $\reg(t)$ as well. Now, if $|r_i(t)-b_i(t)^T\mu| \le R$, for all $i$, then $|\reg'(t)|\le 2R$ and $|Y_t-Y_{t-1}|\le \frac{6}{p} \frac{\tdev^2}{\cof} + 2R$, and we can apply Azuma-Hoeffding inequality exactly as in the proof of Theorem \ref{th:dep} to obtain regret bounds of the same order as Theorem \ref{th:dep} for the new definition. The results extend to the more general $R$-sub-Gaussian condition on $r_i(t)$, using a simple extension of Azuma-Hoeffding inequality; we omit the proof of that extension.

\end{document}